\documentclass{article}

\usepackage{microtype}
\usepackage{graphicx}
\usepackage{subfigure}
\usepackage{booktabs} 
\usepackage{amsthm}
\usepackage{enumitem}

\usepackage{hyperref}

\usepackage[accepted]{icml2020}

\usepackage{amsmath,amsfonts,bm}

\def\eqref#1{equation~\ref{#1}}

\def\1{\bm{1}}

\def\valpha{{\bm{\alpha}}}

\def\vm{{\bm{m}}}

\def\vx{{\bm{x}}}

\def\vz{{\bm{z}}}

\def\mM{{\bm{M}}}

\DeclareMathAlphabet{\mathsfit}{\encodingdefault}{\sfdefault}{m}{sl}
\SetMathAlphabet{\mathsfit}{bold}{\encodingdefault}{\sfdefault}{bx}{n}

\def\gS{{\mathcal{S}}}

\newcommand{\E}{\mathbb{E}}

\newcommand{\R}{\mathbb{R}}

\usepackage{cleveref}

\newtheorem{proposition}{Proposition}

\newtheorem{theorem}{Theorem}

\Crefname{equation}{Eq.}{Eqns.}
\Crefname{figure}{Fig.}{Figs.}
\creflabelformat{equation}{(#2#1#3)}

\icmltitlerunning{Memory Augmented Generative Adversarial Networks for Anomaly Detection}

\begin{document}

\twocolumn[
\icmltitle{Memory Augmented Generative Adversarial Networks for Anomaly Detection}



\begin{icmlauthorlist}
\icmlauthor{Ziyi Yang}{me}
\icmlauthor{Teng Zhang}{mse}
\icmlauthor{Iman Soltani Bozchalooi}{ford}
\icmlauthor{Eric Darve}{me,icme}
\end{icmlauthorlist}

\icmlaffiliation{me}{Department of Mechanical Engineering, Stanford University}
\icmlaffiliation{mse}{Department of Management Science and Engineering, Stanford University}
\icmlaffiliation{icme}{Institute for Computational and Mathematical Engineering, Stanford University}
\icmlaffiliation{ford}{Ford Greenfield Labs}

\icmlcorrespondingauthor{Ziyi Yang}{ziyi.yang@stanford.edu}
\icmlkeywords{Anomaly Detection}

\vskip 0.3in
]



\printAffiliationsAndNotice{} 

\begin{abstract}
In this paper, we present a memory-augmented algorithm for anomaly detection. Classical anomaly detection algorithms focus on learning to model and generate normal data, but typically guarantees for detecting anomalous data are weak. The proposed Memory Augmented Generative Adversarial Networks (MEMGAN) interacts with a memory module for both the encoding and generation processes. Our algorithm is such that most of the \textit{encoded} normal data are inside the convex hull of the memory units, while the abnormal data are isolated outside. Such a remarkable property leads to good (resp.\ poor) reconstruction for normal (resp.\ abnormal) data and therefore provides a strong guarantee for anomaly detection. Decoded memory units in MEMGAN are more interpretable and disentangled than previous methods, which further demonstrates the effectiveness of the memory mechanism. Experimental results on twenty anomaly detection datasets of CIFAR-10 and MNIST show that MEMGAN demonstrates significant improvements over previous anomaly detection methods.
\end{abstract}

\section{Introduction}
\label{intro}
Anomaly detection is the identification of abnormal events, items or data. It has been widely used in many fields, e.g., fraud detection \citep{fraud}, medical diagnosis \citep{schlegl2017unsupervised} and network intrusion detection \citep{net}. Anomaly detection usually is formulated as an unsupervised learning problem where only normal data are available for training while anomalous data are not known a priori (in this paper, ``normal'' does not refer to Gaussian distributions unless specified). Deep learning \citep{deepbook, deepnature} has achieved great success in fields like computer vision \citep{resnet} and natural language processing \citep{bert}. Efforts have been made to apply deep learning to anomaly detection \citep{erfani2016high, schlegl2017unsupervised, chen2017outlier, ruff2018deep} and have shown encouraging results. However, most of existing methods that rely on deep generative models \citep{gan} focus on modeling the normal data distribution using reconstruction heuristics or adversarial training. The proposed optimization objectives are not specifically designed for anomaly detection or the reasons why the proposed models can detect anomalies are ambiguous. Memory-augmented neural network (MANN) is an emerging and powerful class of deep learning algorithms in which the capabilities of networks are extended by external memory resources with an attention mechanism \citep{ntm, metantm, mem_iclr}. Only recently has MANN been applied to anomaly detection; for example, MEMAE \citep{memae} presents a memory augmented autoencoder. Although MANN has shown encouraging results in anomaly detection, theoretical reasons why memory augmentation helps anomaly detection are still unclear.

In this paper, we propose a memory augmented generative adversarial networks (MEMGAN) for anomaly detection. MEMGAN is built upon a bidirectional GAN that includes a generator $G$, encoder $E$, and discriminator $D_{\vx\vz}$. MEMGAN also explicitly maintains an external memory module to store the features of normal data. The key contribution is our definition of the loss function which includes a memory projection loss and mutual information loss. As a result, the memory units form a special geometric structure. After training, the encoder maps the normal data to a convex set in the encoded space, and our loss functions ensure that memory modules lie on the boundary of the convex hull of the encoded normal data. This results in high quality memory modules, that tend to be disentangled and interpretable compared to previous algorithms. In our benchmarks, this leads to higher quality memory modules and much lower detection error.


We conduct detailed experimental and theoretical analysis on how and why MEMGAN can detect anomalies. Theoretical analysis shows that the support of encoded normal data is a convex polytope and the optimal memory units are the vertices of the polytope (assuming hyperparameters for the network are chosen appropriately). This conclusion agrees with a remarkable observation that the encoded normal data reside inside the convex hull of the memory units. The boundary also ensures that encoded abnormal data lie outside the convex polytope. Visualization of decoded memory units demonstrates that MEMGAN successfully capture key features of the normal data. Decoded memory units of MEMGAN display significantly higher quality than previous memory augmented methods.

We evaluate MEMGAN on twenty real-world datasets. The experiments on MNIST show that MEMGAN achieves significant and consistent improvements as compared to baselines (including DSVDD). On CIFAR-10 datasets, MEMGAN shows performance on par with DSVDD and superior to the other models we tested. Ablation study further confirms the effectiveness of MEMGAN.


\section{Related Work}

One major type of anomaly detection algorithms is generative models (autoencoders, variational autoencoders, GANs, etc.) that learn the distribution of normal data. Generative Probabilistic Novelty Detection (GPND) \citep{pidhorskyi2018generative} adopts an adversarial autoencoder to create a low-dimensional representation and compute how likely one sample is anomalous. In \cite{an2015variational} researchers train a variational autoencoder obtain the reconstruction probability through Monte-Carlo sampling as the anomaly score. Anomalyn Detection GAN (ADGAN) \citep{schlegl2017unsupervised} trains a regular GAN model on normal data and projects data back to latent space by gradient descent to compute the reconstruction loss as anomaly score function. Consistency-based anomaly detection (ConAD) (CONAD) \citep{nguyen2019anomaly} employs multiple-hypotheses networks to model normal data distributions. Regularized Cycle-Consistent GAN (RCGAN) \citep{rcgan} introduces a regularized distribution to bias the generation of the bidirectional GAN towards normal data and provides theoretical support for the guarantee of detection.

Other representative anomaly detection algorithms include One-Class Support Vector Machines (OC-SVM) \cite{scholkopf2000support} that models a distribution to encompasses normal data, and those are out-of-distribution are labeled as abnormal. Researchers also utilizes the softmax score distributions from a pretrained classifier by temperature scaling and perturbing inputs \cite{liang2017enhancing}. Deep Support Vector Data Description (DSVDD) jointly trains networks to extract common factors of variation from normal data together, together with a data-enclosing hypersphere in output space \cite{ruff2018deep}. 

Memory augmented neural networks have been proved effective in many applications, e.g., questions answering, graph traversal tasks and few-shot learning \citep{hybrid, oslmem}. In such models, neural networks have access to external memory resources and interact with them by reading and writing operations \citep{ntm} to resemble Von Neumann architecture in modern computers. Recently, \citet{memae} proposes to combine memory mechanism and autoencoders for anomaly detection. However, as mentioned in the introduction, the guarantees for anomaly detection are weak for most of the models above, or their training objective functions are not specifically designed for anomaly detection.

\section{Problem Statement}

Anomaly detection refers to the task of identifying anomalies from what are believed to be normal. The normal data can be described by a probability density function $q(\vx)$ \citep{rcgan}. In the training phase, we want to learn an anomaly score function $A(\vx)$ given only normal data. The function $A(\vx)$ is expected to assign larger score for anomalous data than normal ones.

Deep generative models are capable of learning the normal data distribution $q(\vx)$, e.g., generative adversarial networks (GANs) proposed in \citet{gan}. GAN trains a discriminator $D$ and a generator $G$ such that $D$ learns to distinguish real data sampled using $q(\vx)$ from synthetic data generated by $G$ using a random distribution $p(\vz)$. The minmax optimization objective of GANs is given as:
\begin{equation}
\label{eq:gan}
\begin{aligned}
\min_{G}\max_{D}V(D, G) &= \E_{\vx\sim q(\vx)}[\log D(\vx)]\\
&\; + \; \E_{\vz\sim p(\vz)}[\log(1 - D(G(\vz)))]
\end{aligned}
\end{equation}
where $p(\vz)$ denotes a random distribution such as uniform distributions. As shown in \citet{gan}, the optimal generator distribution $p(\vx)$ matches with the data distribution, i.e., $p(\vx) = q(\vx)$.

In order to provide a convenient projection from the data space to the latent space, bidirectional GAN proposed in \citet{ali} and \citet{afl} includes an encoder $E$ network with the following optimization objective:
\begin{align}
\min_{E, G}\max_{D_{\vx\vz}}V_\text{ALI}&(D_{\vx\vz}, G, E) = \E_{\vx\sim q(\vx)}[\log D_{\vx\vz}(\vx, E(\vx))] \notag \\
&\; + \; \E_{\vz\sim p(\vz)}[\log(1 - D_{\vx\vz}(G(\vz), \vz))]
\label{eq:ali}
\end{align}
where $D_{\vx\vz}$ denotes a discriminator with dual inputs data $\vx$ and the latent variable $\vz$. Its output is the probability that $\vx$ and $\vz$ are from the real data joint distribution $q(\vx, \vz)$. ALI attempts to match the generator joint distribution $q(\vx, \vz) = q(\vx)q(\vz|\vx)$ and the data joint distribution $p(\vx, \vz) = p(\vz)p(\vx|\vz)$. It follows that
\begin{theorem}
\label{thm:match}
The optimum of the encoder, generator and discriminator in ALI is a saddle point of \Cref{eq:ali} if and only if the encoder joint distribution matches with the generator joint distribution, i.e., $q(\vx, \vz) = p(\vx, \vz)$.
\end{theorem}

\section{Methodology}
\subsection{Memory Augmented Bidirectional GAN}

The fundamental building block for MEMGAN is the bidirectional GAN proposed in ALI. Besides the bidirectional projection in ALI, MEMGAN also maintains an external memory $\mM \in \R^{n\times d}$, where $n$ represents the number of memory slots and $d$ denotes the dimension of memory unit. In other words, each row of the memory matrix $\mM$ is a memory unit.

Specifically, the adversarial loss in MEMGAN is as follows:
\begin{align}
\min_{E, G}\max_{D_{\vx\vz}} V_\text{MEM}&(D_{\vx\vz}, G, E) = \E_{\vx\sim q(\vx)}[\log D_{\vx\vz}(\vx, E(\vx))] \notag\\
&\; + \; \E_{\vz\sim \mM}[\log(1 - D_{\vx\vz}(G(\vz), \vz))]
\label{eq:adv_loss}
\end{align}
where $\vz \sim \mM$ denotes sampling the latent variable $\vz$ as a linear combination of the memory units in $\mM$ \textbf{with positive coefficients,} i.e., $\vz$ is a convex combination of rows in $\mM$. Note that this step is distinct from classical GANs that latent variables are sampled from a random distribution, e.g. Gaussian distribution. We will see in the next section that sampling from memory units confines the encoded normal data in a convex polytope and makes memory units more interpretable, which leads to strong guarantee for anomaly detection.

\subsection{Cycle Consistency}

The cycle consistency is a desirable property in bidirectional GANs \citep{alice} and data translation \citep{unpaired}. It requires that a data example $\vx$ matches with its reconstruction $G(E(\vx))$, which can be fulfilled by minimizing the norm of difference between $\vx$ and its reconstruction \citep{unpaired} or by adversarial training through a discriminator \citep{alice, rcgan}. In MEMGAN, \textbf{the reconstruction error} will be used as the anomaly score. 

In order to guarantee a good reconstruction for normal data $\vx\sim q(\vx)$, MEMGAN enforces the cycle consistency by minimizing the reconstruction error. One key difference is instead of directly generating from the encoded $E(\vx)$, MEMGAN applies a linear transformation in the latent space using an attention algorithm with the external memory $\mM$. The cycle consistency loss in MEMGAN is given as follows:
\begin{equation}
\label{eq:recon_loss}
l_{\text{cyc}} = \E_{\vx\sim q(\vx)}||\vx - G(P(E(\vx)))||_2 
\end{equation}
where $P$ denotes a ``projection'' of the encoded $E(\vx)$ onto the subspace spanned by memory units (the rows of $\mM$). The exact equation is
\begin{equation}
\begin{aligned}
\label{eq:proj}
P(\vz) &= \mM^T\valpha \\
\end{aligned}
\end{equation}
with $\valpha = \textrm{softmax}(\mM\vz)$.

The reconstruction loss is also used to update memory units such that the features of normal data are designed to minimize the reconstruction error.

\subsection{Memory Projection Loss}

To ensure that the linear combinations of memory units can represent the encodings of normal data, we include a third optimization objective, the memory projection loss:
\begin{equation}
\label{eq:enc}
l_{\text{proj}} = \E_{\vx\sim q(\vx)}\|E(\vx) - P(E(\vx))\|_2
\end{equation}
As explained later, the memory projection loss together with mutual information loss below constitutes a strong guarantee for anomaly detection. We also conduct an ablation study in \Cref{sec:dis} to show that without $l_{\text{proj}}$, memory units tend to cluster around a sub-region in the latent space and fail to enclose the encoded normal data.

\subsection{Mutual Information Loss}

In order to learn disentangled and interpretable memory units and maximize the mutual information between memory units and the normal data, we introduce the mutual information loss similar to InfoGAN \citep{infogan}. First we randomly sample a vector $\valpha$, with positive coefficients and such that the sum of its elements is 1 (by applying softmax to a random vector); then we use $\valpha$ to compute a linear combination of the memory units, denoted as $\vz$. $\vz$ is mapped to data space by the generator and encoded back as $\vz' = E(G(\vz))$. $\vz'$ is then projected onto the memory units to compute new projection coefficients $\valpha'$. Finally, the mutual information loss is defined as the cross entropy between $\valpha'$ and $\valpha$:
\begin{equation}
\begin{aligned}
\label{eq:mi}
\vz &= \mM^T \valpha\\
\vz' &= E(G(\vz))\\
\valpha'& = \textrm{softmax}(\mM\vz') \\
l_{\text{mi}} &= -\sum_{i = 1}^{n}\valpha_i\log\valpha'_i
\end{aligned}
\end{equation}
where $n$ is the number of memory units. $l_{\text{mi}}$ essentially ensures that the structural information is consistent between a sampled memory information $\vz$ and the generated $G(\vz)$.

We provide the complete and detailed training process of MEMGAN in \Cref{alg:train}. The memory matrix $\mM$ is optimized as a trainable variable and its gradient is computed from $l_{\text{cyc}}, l_{\text{proj}}$ and $l_{\text{mi}}$. Note in order to avoid the vanishing gradient problem \citep{gan}, the generator is updated to minimize $-\log D_{\vx\vz}(G(\vz), \vz)$ instead of $\log (1 - D_{\vx\vz}(G(\vz), \vz)$. After MEMGAN is trained on normal data, in the test phase, the anomaly score for $\vx$ is defined as 
\begin{equation}
A(\vx) = \|\vx - G(P(E(\vx)))\|_2
\label{eq:ano_score}
\end{equation}
This function measures how well an example can be reconstructed by generating from $E(\vx)$'s projection onto the subspace spanned by the memory units. In the next section, we will further explain the effectiveness of this anomaly score.

\begin{algorithm}[tb]
   \caption{The training process of MEMGAN}
   \label{alg:train}
\begin{algorithmic}
   \STATE {\bfseries Input:} a set of normal data $\vx$, the encoder $E$, the generator $G$, the discriminator $D_{\vx\vz}$ and the memory matrix $\mM$.
   \FOR{number of epochs}
   \FOR{$k$ steps}
   \STATE Sample a minibatch of $m$ normal data from $\vx$.
   \STATE Sample a minibatch of $m$ random convex combination of memory units $\vz$.
   \STATE Update the discriminator using its stochastic gradient:
    \begin{equation*}
    \begin{split}
    \nabla_{\theta_{D_{\vx\vz}}}&\frac{1}{m}\sum_{i = 1}^{m}[-\log D_{\vx\vz}(\vx^{(i)}, E(\vx^{(i)})) \\
    &- \log(1 - D_{\vx\vz}(G(\vz^{(i)}), \vz^{(i)})]
    \end{split}
    \end{equation*}
    \STATE Compute $l^{(i)}_{\text{cyc}}$ and $l^{(i)}_{\text{proj}}$ with $\vx^{(i)}$, compute $l^{(i)}_{\text{mi}}$ with $\vz^{(i)} $ for $i = 1\dots m$.
    \STATE Update the Encoder $E$ using its stochastic gradient:
    \begin{equation*}
    \nabla_{\theta_{G}}\frac{1}{m}\sum_{i = 1}^{m}[\log D_{\vx\vz}(\vx^{(i)}, E(\vx^{(i)})) + l^{(i)}_{\text{cyc}} + l^{(i)}_{\text{proj}} + l^{(i)}_{\text{mi}}]
    \end{equation*}
    \STATE Update the Generator $G$ using its stochastic gradient:
    \begin{equation*}
    \nabla_{\theta_{G}}\frac{1}{m}\sum_{i = 1}^{m}[-\log D_{\vx\vz}(G(\vz^{(i)}), \vz^{(i)}) + l^{(i)}_{\text{cyc}} + l^{(i)}_{\text{mi}}]
    \end{equation*}
    
    \STATE Update the Memory $\mM$ using its stochastic gradient:
    \begin{equation*}
    \nabla_{\theta_{\mM}}\frac{1}{m}\sum_{i = 1}^{m}(l^{(i)}_{\text{cyc}} + l^{(i)}_{\text{proj}} + l^{(i)}_{\text{mi}})
    \end{equation*}
   \ENDFOR
   \ENDFOR
\end{algorithmic}
\end{algorithm}
\section{The Mechanism of MEMGAN}
Previous literature on anomaly detection mostly focus on demonstrating the effectiveness of models by showing results on test dataset, while efforts on why and how the proposed algorithms work are lacking. In this section, we will present an explanation for the mechanism of MEMGAN. First, we will show that the memory units obtained by MEMGAN successfully capture high quality latent representations of the normal data. Second, by visualizing the memory units and encoded normal \& abnormal data with a 2D projection, we find that the memory units form a convex hull of the normal data and separate them from abnormal data. This property provides a strong guarantee for anomaly detection.

\subsection{What Do Memory Units Memorize?}
 After training MEMGAN on normal data, one natural question to ask is what exactly do the memory units store. In \Cref{fig:mem}, we visualize the decoding of memory units, i.e., inputting the rows of the memory matrix $\mM$ to the generator after training on images of a digit from MNIST dataset (i.e. treat each digit as the normal class). Results from MEMGAN are shown in the first and second row, as well as the first two sub-figures in the third row. Each sub-figure contains eight samples of 50 decoded memory units due to the space limit. Examples of all 50 decoded units are available in the supplementary material. The last sub-figure in the third row is directly taken from the original paper of MEMAE (a previous memory augmented anomaly detection algorithm) which presented four decoded memory units trained on the digit 9.

The memory units from MEMGAN successfully learn latent representations of normal data. The generated images are highly distinct, recognizable and diverse. In contrast, the decoded memory vectors of MEMAE are vague, even the background which should be dark is noisy with gray dots. The superior memory mechanism of MEMGAN should be attributed to our novel adversarial loss and mutual information loss. The former one guarantees that the generated results from the memory units are indistinguishable from normal data judged by $D_{\vx\vz}$. The latter one makes the memory units distinct from one another that prevents slots from memorizing the same thing and spread along the boundary of the convex polytope containing the encoded normal data. The diversity of decoded memory units and the topological property of the memory units space confirm this.

\begin{figure*}[htb]
\centering
  \includegraphics[width=0.95\textwidth]{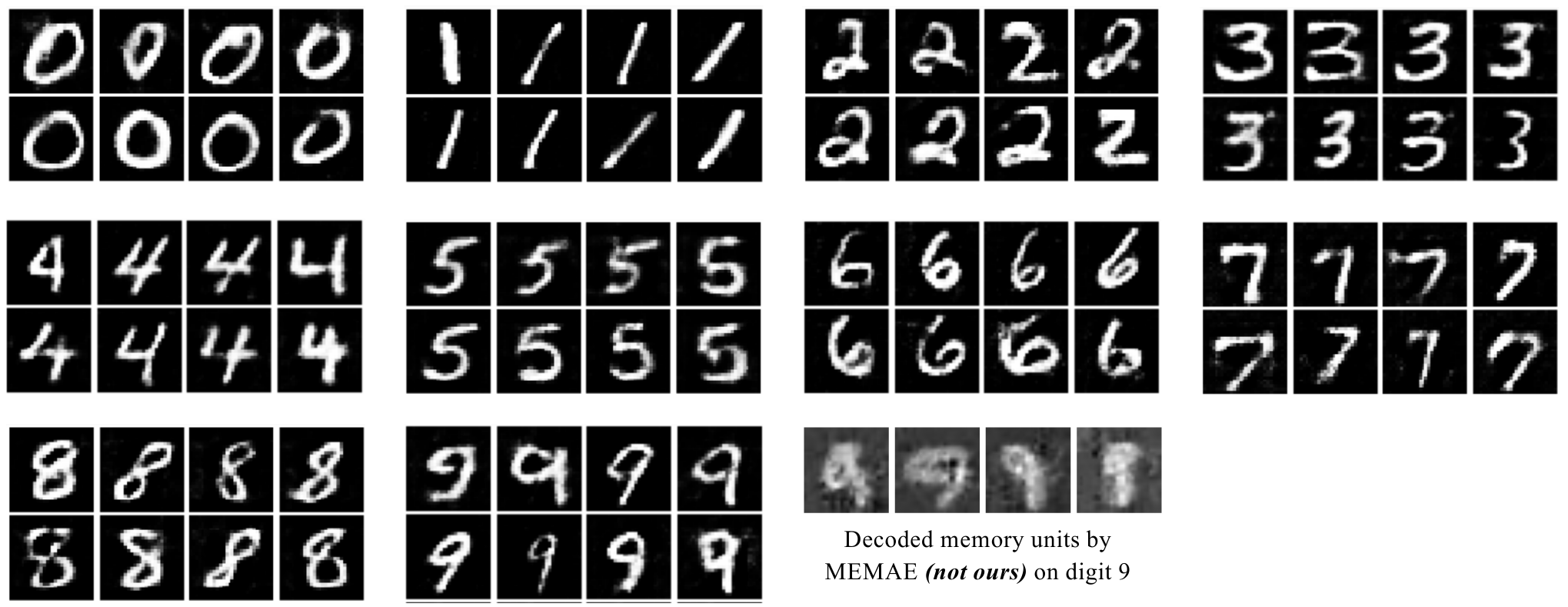}
  \caption{Examples of decoded memory units on MNIST dataset. The sub-figures in the first and second rows, as well as the first two pictures in the third rows are of MEMGAN (ours). The last picture in the third row is directly taken from the original paper of MEMAE (a baseline model also using memory augmentation). The decoded memory units from MEMGAN display more disentangled and recognizable digits than MEMAE.}
  \label{fig:mem}
\end{figure*}

\subsection{Topological Property of the Latent Space in MEMGAN}

Now, we scrutinize the latent space where the memory units $\mM$ live. By optimizing \Cref{eq:adv_loss}, we have the following theorem:
\begin{theorem}
Assume the dimension of the latent variables $d$ and the number of memory units $n$ are large enough. Since $\vz$ is a convex combination of rows in $\mM$, the support of the $\vz$ distribution is a convex polytope $\gS$ in $\R^{d}$, where the vertices of $\gS$ are the memory units. The support of the encoded normal data distribution $E(\vx)$ is also $\gS$.
\end{theorem}
\begin{proof}
 Because $\vz$ is convex combination of memory units, the support of $\vz$ is a convex polytope. As claimed in \Cref{thm:match}, at the optimality of \Cref{eq:ali}, the joint distribution $p(\vx, \vz)$ and $q(\vx, \vz)$ matches. Therefore the marginal distributions, including the supports, should also match, i.e., $E(\vx)$ and $\vz$.
 \end{proof}
 




Given a fixed encoder and generator, consider the memory units optimized using the definition of projection loss and mutual information loss. As mentioned previously, the memory projection loss in \Cref{eq:enc} ensures that the memory units can effectively represent the encoded normal data. The mutual information loss ensures the sampled memory vectors are within the support of $E(\vx)$. Denote the support of $E(\vx)$ as $\gS$. For the purpose of our theoretical analysis, we therefore replace the mutual information loss by the term $\int_{\valpha} 1 - \mathbf{1}_\gS(\sum_i \valpha_i \vm_i) d\valpha$ which expresses the fact that points inside the convex hull of the memory are mapped to the normal data. Note that $\valpha > 0$ and $\sum_i \valpha_i  = 1$. $\mathbf{1}_\gS$ denotes the indicator function of $\gS$. We then prove the following
\begin{proposition}
For the convex hull $\gS$ of a set of $\vz_j$ points, the optimal memory units $\vm$ that minimize the following function ($\valpha > 0$ and $\sum_i \valpha_i  = 1$) are the vertices of $\gS$:
\begin{equation}
\label{eq:poly}
\begin{split}
\min_{\vm} \int_{\valpha} 1 & - \mathbf{1}_\gS(\sum_i \alpha_i \vm_i) d\valpha \\
&+ \sum_{j}\min_{\valpha}\|\sum_{i}\alpha_i \vm_i - \vz_j\|_2
\end{split}
\end{equation}
assuming the number of rows in $\mM$ is not smaller than the number of vertices.

\end{proposition}

\begin{proof}
If $\vm$ are the vertices of the convex hull, \Cref{eq:poly} reaches its global minimum value 0. The solution is unique, since perturbing the memory units away from the convex set $\gS$ causes the first term to be positive. Perturbing the memory units inside $\gS$ (or away from the vertices, on the boundary) causes the second term to be non-zero.
\end{proof}

We also ran numerical experiments that confirm the above proposition. As shown in \Cref{fig:poly}, given a set of encoded normal data (orange) bounded in $[-1, 1]^2 \in \R^2$, the obtained memory units (red) by optimizing \Cref{eq:poly} are approximately the vertices of convex hull.

\begin{figure}
\centering
  \includegraphics[width=0.9\columnwidth]{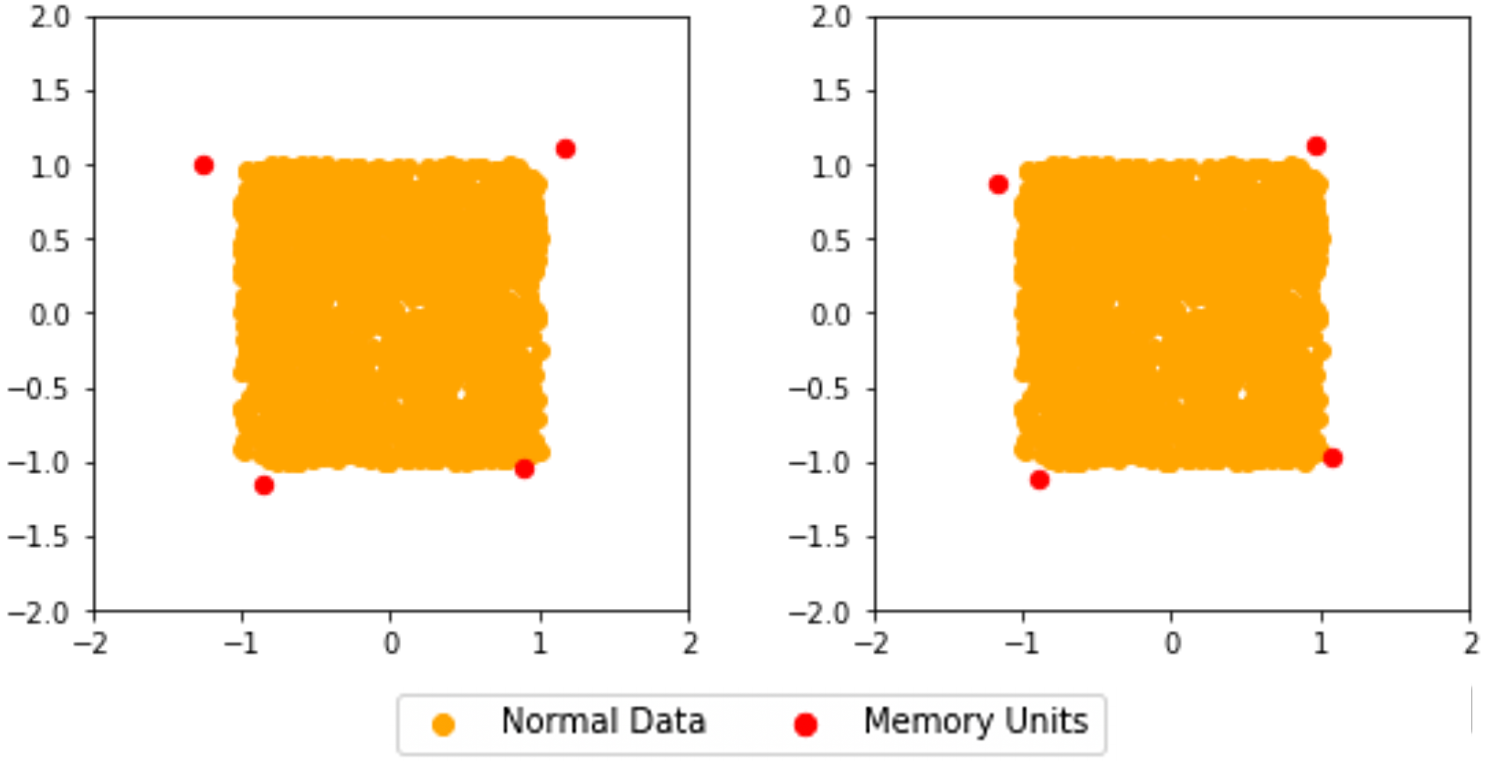}
  \caption{The orange dots denote encoded normal data for training bounded in $[-1, 1]^2$. The red dots represent obtained memory units by optimizing \Cref{eq:poly}. The memory units approximate with the vertices of the convex hull of encoded data.}
  \label{fig:poly}
\end{figure}

\begin{figure*}
\centering
  \includegraphics[width=0.95\textwidth]{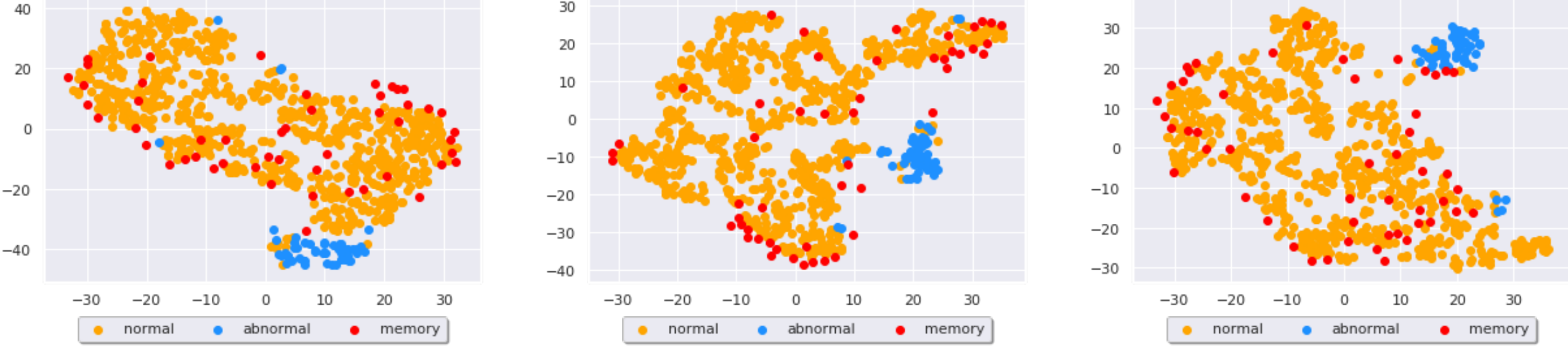}
  \caption{Projections of memory units, normal data and abnormal data onto 2D for MNIST dataset. The memory units (red) form a convex hull for the normal data (orange) and separates them from the abnormal data (blue). This ``isolation'' property yields a strong guarantee for anomaly detection of MEMGAN.}
  \label{fig:hull}
\end{figure*}

Based on the analysis above, we can conclude that at optimality for \Cref{eq:poly}, $E(\vx)$ lies within a convex polytope in the latent space, whose vertices are the memory units. In other words, the training phase of MEMGAN is learning a low dimensional linear representation of the normal data, as well as the memory units which are the vertices of the convex hull. In \Cref{fig:hull}, we visualize the memory units, encoded normal and abnormal data on MNIST dataset on 2D by applying t-SNE (T-distributed Stochastic Neighbor Embedding) algorithm \citep{tsne}. We find that the memory units enclose most of the encoded normal data and lie near the boundary of the cluster of normal data. Meanwhile the memory units are able to separate the normal from the abnormal data. This figure is consistent with the theoretical results although making stronger conclusions is difficult because of the non-linearity in the low-dimensional t-SNE projection. In particular, due to the t-SNE visualization, points on the boundary may appear inside the domain.


\subsection{Detection of Anomalous Data}

Abnormal data are separated from normal ones by residing outside of the convex hull, which eventually leads to larger reconstruction losses for anomalies in the data space (compared with normal data). This is confirmed by our theoretical result and partially illustrated by \Cref{fig:hull}. Consider an abnormal data $\vx_{\text{abn}}$, the projection of its encoding $P(E(\vx_{\text{abn}}))$ will be inside of the convex hull of the memory units, while $E(\vx_{\text{abn}})$ is outside of it. Therefore, the difference between $G(P(E(\vx_{\text{abn}})))$ and $\vx_\text{abn}$ is expected to be large. In contrast, for a normal data $\vx_{\text{nor}} \sim q(\vx)$, since $E(\vx_{\text{nor}})$ are inside the convex hull of memory units, we expect a smaller difference (compared with abnormal data) between $G(P(E(\vx_{\text{nor}})))$ and $\vx_{\text{nor}}$. This distinction provides a strong guarantee for anomaly detection. This claim is also validated by the superior performance of MEMGAN.

\section{Experiments}
\begin{table*}[t]
\caption{Experiment results on MNIST dataset by treating each class as the normal one, evaluated by AUROC. Performance with highest mean is in bold. MEMGAN achieves the highest mean performance in seven out of ten experiments. MEMGAN also has the highest average performance on ten datasets. Standard deviations are also included.}
\vskip 0.1in
\centering
\begin{tabular}{c|ccccccccc}
\toprule
Normal & OCSVM & DCAE & IF & ADGAN & KDE & DSVDD & MEMAE & \textbf{MEMGAN}\\
\midrule
0 & 98.6$\pm$0.0 & 97.6$\pm$0.7 & 98.0$\pm$0.3 & 96.6$\pm$1.3 & 97.1$\pm$0.0 & 98.0$\pm$0.7 & \textbf{99.3$\pm$0.1} & \textbf{99.3$\pm$0.1}\\
1 & 99.5$\pm$0.0 & 98.3$\pm$0.6 & 97.3$\pm$0.4 & 99.2$\pm$0.6 & 98.9$\pm$0.0 & 99.7$\pm$0.1 & 99.8$\pm$0.0 & \textbf{99.9$\pm$0.0}\\
2 & 82.5$\pm$0.1 & 85.4$\pm$2.4 & 88.6$\pm$0.5 & 85.0$\pm$2.9 & 79.0$\pm$0.0 & 91.7$\pm$0.8 & 90.6$\pm$0.8& \textbf{94.5$\pm$0.1}\\
3 & 88.1$\pm$0.0 & 86.7$\pm$0.9 & 89.9$\pm$0.4 & 88.7$\pm$2.1 & 86.2$\pm$0.0 & 91.9$\pm$1.5 & 94.7$\pm$0.6 & \textbf{95.7$\pm$0.4}\\
4 & 94.9$\pm$0.0 & 86.5$\pm$2.0 & 92.7$\pm$0.6 & 89.4$\pm$1.3 & 87.9$\pm$0.0 & 94.9$\pm$0.8 & 94.5$\pm$0.4 & \textbf{96.1$\pm$0.4}\\
5 & 77.1$\pm$0.0 & 78.2$\pm$2.7 & 85.5$\pm$0.8 & 88.3$\pm$2.9 & 73.8$\pm$0.0 & 88.5$\pm$0.9 & \textbf{95.1$\pm$0.1} & 93.6$\pm$0.3\\
6 & 96.5$\pm$0.0 & 94.6$\pm$0.5 & 95.6$\pm$0.3 & 94.7$\pm$2.7 & 87.6$\pm$0.0 & 98.3$\pm$0.5 & 98.4$\pm$0.5 & \textbf{98.6$\pm$0.1}\\
7 & 93.7$\pm$0.0 & 92.3$\pm$1.0 & 92.0$\pm$0.4 & 93.5$\pm$1.8 & 91.4$\pm$0.0 & 94.6$\pm$0.9 & 95.4$\pm$0.2 & \textbf{96.2$\pm$0.2}\\
8 & 88.9$\pm$0.0 & 86.5$\pm$1.6 & 89.9$\pm$0.4 & 84.9$\pm$2.1 & 79.2$\pm$0.0 & \textbf{93.9$\pm$1.6} & 86.9$\pm$0.5 & 93.5$\pm$0.1\\
9 & 93.1$\pm$0.0 & 90.4$\pm$1.8 & 93.5$\pm$0.3 & 92.4$\pm$1.1 & 88.2$\pm$0.0 & 96.5$\pm$0.3 & \textbf{97.3$\pm$0.2} & 95.9$\pm$0.1\\
\midrule
Average & 91.3 & 89.7 & 92.3 & 91.4 & 87.0 & 94.8 & 95.2 & \textbf{96.5}\\
\bottomrule
\end{tabular}
\label{tab:mnist}
\end{table*}

\begin{table*}[t]
\caption{Anomaly detection on CIFAR-10 dataset. Performance with highest mean is in bold. In three out of ten datasets, MEMGAN has the highest performance. MEMGAN achieves the highest (together with DSVDD) average performance on all ten datasets.}
\vskip 0.1in
\centering
\begin{tabular}{c|cccccccccc}
\toprule
Normal & DSVDD & DSEBM & DAGMM & IF & ADGAN & ALAD & MEMAE & \textbf{MEMGAN}\\
\midrule
airplane & 61.7$\pm$4.1 & 41.4$\pm$2.3 & 56.0$\pm$6.9 & 60.1$\pm$0.7 & 67.1$\pm$2.5 & 64.7$\pm$2.6 & 66.5$\pm$0.9 & \textbf{73.0$\pm$0.8}\\
auto. & \textbf{65.9$\pm$2.1} & 57.1$\pm$2.0 & 48.3$\pm$1.8 & 50.8$\pm$0.6 & 54.7$\pm$3.4 & 38.7$\pm$0.8 & 36.2$\pm$0.1 & 52.5$\pm$0.7\\
bird & 50.8$\pm$0.8 & 61.9$\pm$0.1 & 53.8$\pm$4.0 & 49.2$\pm$0.4 & 52.9$\pm$3.0 & \textbf{67.0$\pm$0.7} & 66.0$\pm$0.1 & 62.1$\pm$0.3\\
cat & 59.1$\pm$1.4 & 50.1$\pm$0.4 & 51.2$\pm$0.8 & 55.1$\pm$0.4 & 54.5$\pm$1.9 & \textbf{59.2$\pm$1.1} & 52.9$\pm$0.1 & 55.7$\pm$1.1\\
deer     & 60.9$\pm$1.1 & 73.3$\pm$0.2 & 52.2$\pm$7.3 & 49.8$\pm$0.4 & 65.1$\pm$3.2 & 72.7$\pm$0.6 & 72.8$\pm$0.1 & \textbf{73.9$\pm$0.9}\\
dog      & \textbf{65.7$\pm$0.8} & 60.5$\pm$0.3 & 49.3$\pm$3.6 & 58.5$\pm$0.4 & 60.3$\pm$2.6 & 52.8$\pm$1.2 & 52.9$\pm$0.2 & 64.7$\pm$0.5\\
frog     & 67.7$\pm$2.6 & 68.4$\pm$0.3 & 64.9$\pm$1.7 & 42.9$\pm$0.6 & 58.5$\pm$1.4 & 69.5$\pm$1.1 & 63.7$\pm$0.4 & \textbf{72.8$\pm$0.7}\\
horse    & \textbf{67.3$\pm$0.9} & 53.3$\pm$0.7 & 55.3$\pm$0.8 & 55.1$\pm$0.7 & 62.5$\pm$0.8 & 44.8$\pm$0.4 & 45.9$\pm$0.1 & 52.5$\pm$0.5\\
ship     & \textbf{75.9$\pm$1.2} & 73.9$\pm$0.3 & 51.9$\pm$2.4 & 74.2$\pm$0.6 & 75.8$\pm$4.1 & 73.4$\pm$0.4 & 70.1$\pm$0.1 & 74.1$\pm$0.3\\
truck    & \textbf{73.1$\pm$1.2} & 63.6$\pm$3.1 & 54.2$\pm$5.8 & 58.9$\pm$0.7 & 66.5$\pm$2.8 & 39.2$\pm$1.3 & 38.1$\pm$0.1 & 65.6$\pm$1.6\\
\midrule
Average     & \textbf{64.8} & 60.4 & 54.4 & 55.5 & 61.8 &  59.3 & 56.5 & \textbf{64.8}\\
\bottomrule
\end{tabular}
\label{tab:cifar-10}
\end{table*}

In this section, we evaluate MEMGAN on MNIST and CIFAR-10 datasets. We test on computer vision datasets since images are high dimensional data that can better evaluate MEMGAN's ability. We create anomaly detection datasets from image classification benchmarks by regarding each class as normal for each dataset. The evaluation metric is the area under a receiver operating characteristic curve (AUROC). We first start with the introduction of baseline models.
\subsection{Baseline Models}
In our experiments with real-world dataset, we compare MEMGAN with the following baseline models:

\textbf{Isolation Forests (IF)} ``isolates'' data by randomly selecting a feature and then randomly selecting a split value between the maximum and minimum values of the selected feature to construct trees \citep{liu2008isolation}. The averaged path length from the root node to the example is a measure of normality.

\textbf{Anomaly Detection GAN (ADGAN)} trains a DCGAN on normal data and compute the corresponding latent variables by minimizing the reconstruction error and feature matching loss using gradient descent \citep{schlegl2017unsupervised}. The anomaly score is the reconstruction error.

\textbf{Adversarially Learned Anomaly Detection (ALAD)} trains a bidirectional GAN framework with an extra discriminator to achieve cycle consistency in both data and latent space \citep{Zenati2018AdversariallyLA}. The anomaly score is the feature matching error.

\textbf{Deep Structured Energy-Based Model (DSEBM)} trains deep structured EBM with a regularized autoencoder \citep{zhai2016deep}. The energy score is used as the anomaly score function.

\textbf{Deep Autoencoding Gaussian Mixture Model (DAGMM)} trains a Gaussian Mixture Model for density estimation together with a encoder \citep{zong2018deep}. The probability given by the Gaussian mixture are defined as the anomaly score.

\textbf{Deep Support Vector Data Description (DSVDD)} minimizes the volume of a hypersphere that encloses the encoded representations of data \citep{ruff2018deep}. The Euclidean distance of the data the center of hypersphere is regarded as the anomaly score.

\textbf{One Class Support Vector Machines (OC-SVM)} is a kernel-based method that learns a decision function for novelty detection \cite{scholkopf2000support}. It classifies new data as similar or different to the normal data.

\textbf{Kernel Density Estimation (KDE)} models the normal data probability density function $q(\vx)$ in a non-parametric way \citep{kde}. The anomaly score can be the negative of the learned data probability.

\textbf{Deep Convolutional Autoencoder (DCAE)} trains a regular autoencoder with convolutional neural network \citep{dcae} on normal data. The anomaly score is defined as the reconstruction error.

\textbf{Memory Augmented Autoencoder (MEMAE)} trains an autoencoder with an external memory \citep{memae}. The input to the decoder is a sample of memory vectors. During training, the weight coefficients are sparsified by a threshold.

\subsection{MNIST Dataset}
We test on the MNIST dataset. Ten different datasets are generated by regarding each digit category as the normal class. We use the original train/test split in the dataset. The training set size for one dataset is around 6000. MEMGAN is trained on all normal images in the training set for one class. The evaluation metrics is area under the receiver operating curve (AUROC), averaged on 10 runs. The number of memory units is $n = 50$. The size of one memory unit is $d = 64$. We also explore the effects of number of memory units on the performance on \Cref{sec:dis}. The learning rate is between $10^{-5}$ and $10^{-4}$, varying between different classes. The number of epochs is 7. The configurations of neural networks can be found in the supplementary materials.

MEMGAN shows superior performance compared with all other models, achieving highest mean in seven out of ten classes. MEMGAN also has the best highest average AUROC of ten experiments. MEMGAN clearly outperforms MEMAE and DSVDD in terms of overall performance.

\subsection{CIFAR-10 Dataset}
Similar to the experimental settings in MNIST, we also test MEMGAN on CIFAR-10. Each image category is regarded as the normal class to generate ten distinct datasets. The size of a training dataset is 6000. The number of memory units are chosen to be 100. The size of one memory unit is 256. The learning rate is between $5 \cdot 10^{-5}$ and $10^{-4}$, varying between different classes. The model is trained for 10 epochs. The experiments on MINST and CIFAR-10 are both conducted on a GeForce RTX 2080. The configurations of neural networks are specified in the supplementary material.

Overall, MEMGAN's performance is on a par with previous state-of-the-art models. MEMGAN achieves the highest performance in three out of ten classes and the highest average performance (together with DSVDD). Compared with previous GAN based models (ADGAN and ALAD), MEMGAN exhibits superior performance. Notably MEMGAN outperforms another memory augmented model MEMAE by significant margins. MEMGAN also shows great advantage over non-deep-learning baselines.

\section{Discussion}
\label{sec:dis}
\begin{figure}
\centering
  \includegraphics[width=0.75\columnwidth]{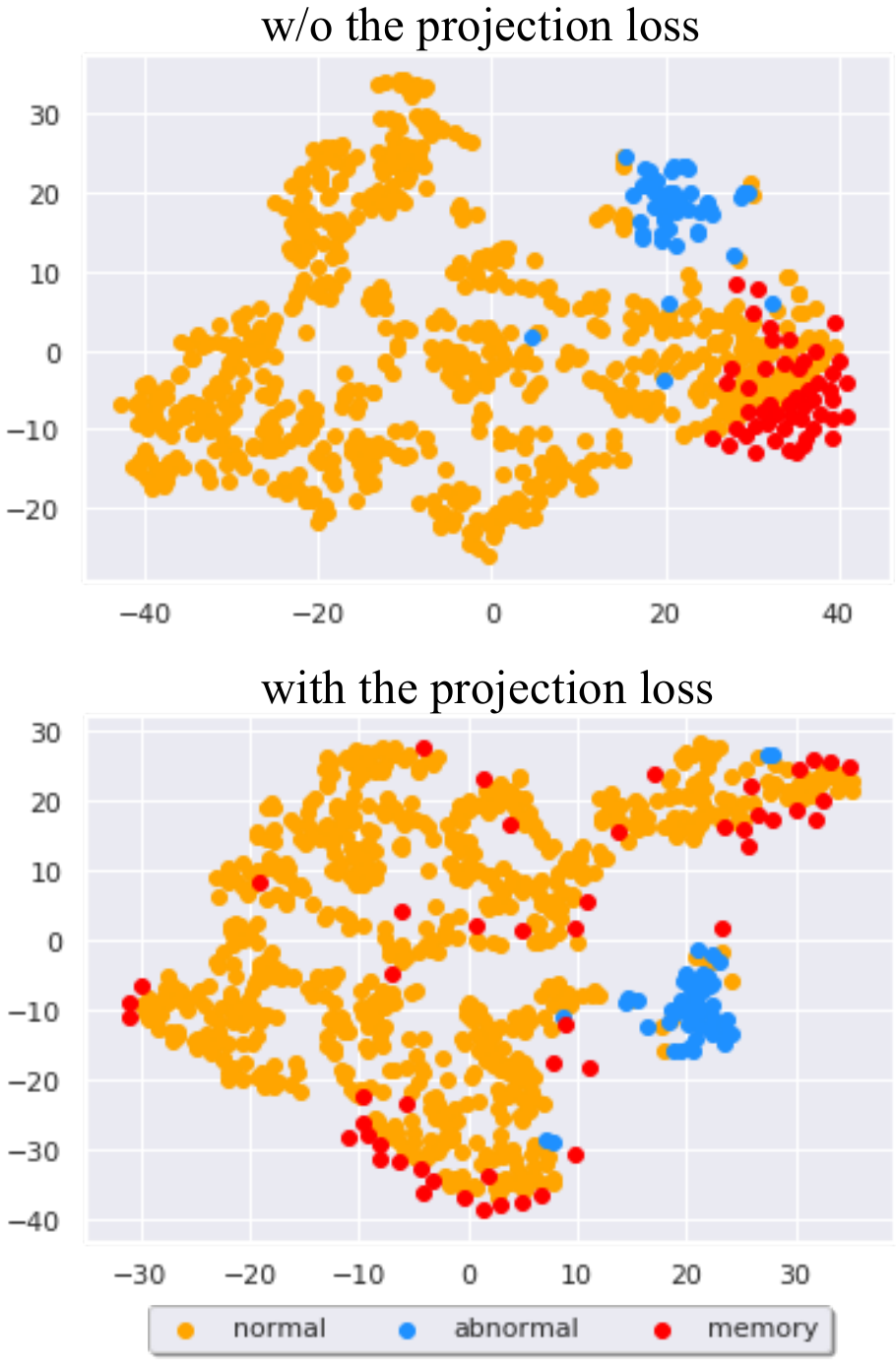}
  \caption{Projections of memory units and data examples with (lower) and without (upper) the memory projection loss. The convex hull of memory units no longer contains the normal data without the memory projection loss.}
  \label{fig:enc_abl}
\end{figure}

\begin{table}
\caption{Sensitivity study on the number of units $n$ tested on MNIST dataset (class 0, 1 and 5). MEMGAN does not show significant sensitivity to $n$.}
\vskip 0.15in
\centering
\begin{tabular}{c|cccc}
\toprule
$n$ & 25 & 50 & 100 & 200\\
\midrule
0 & 99.2$\pm$0.1 & 99.3$\pm$0.1 & 99.1$\pm$0.1 & 99.2$\pm$0.1\\
1 & 99.7$\pm$0.0 & 99.9$\pm$0.0 & 99.8$\pm$0.0 & 99.8$\pm$0.0\\
5 & 93.6$\pm$0.1 & 93.6$\pm$0.3 & 93.2$\pm$0.2 & 93.7$\pm$0.9 \\
\bottomrule
\end{tabular}
\label{tab:sens}
\end{table}

\textbf{Ablation study on the memory projection loss.} We test with and without the memory projection loss in \Cref{eq:enc}, and then visualize the data and memory units in \Cref{fig:enc_abl}. Without the memory projection loss, the memory units tend to cluster together in an area \textbf{with high density of normal data.} In contrast, with the memory projection loss, the memory units are more evenly distributed in the latent space and their convex hull covers most of the encoded normal data. This observation again confirms the effectiveness of the memory projection loss.

\textbf{Sensitivity Study on Number of Memory Units.} We test the performance sensitivity of MEMGAN with respect to the number of memory units.  Results in \Cref{tab:sens} show that MEMGAN is robust to different number of memory units. This means that MEMGAN can effectively learn memory units with different number of memory units.

\section{Conclusion}

In this paper, we introduced MEMGAN, a memory augmented bidirectional GAN for anomaly detection. We proposed a theoretical explanation on MEMGAN's effectiveness and showed that trained memory units enclose encoded normal data. This theory led to the discovery that the encoded normal data reside in the convex hull of the memory units, while the abnormal data are located outside. Decoded memory units obtained by MEMGAN are of much improved quality and disentangled compared to previous methods. Experiments on CIFAR-10 and MNIST further demonstrate quantitatively as well as qualitatively the superior performance of MEMGAN.

\bibliography{icml2020}
\bibliographystyle{icml2020}
\end{document}